\documentclass[11pt]{article} 

\usepackage{fullpage}

\usepackage{color}
\definecolor{mypink}{RGB}{227,29,118}
\definecolor{myblue}{RGB}{0,90,253}
\usepackage{amsmath,amsfonts,amsthm,amssymb,xspace,bm,verbatim,dsfont,mathtools}
\usepackage{hyperref}
\hypersetup{colorlinks=true,citecolor=mypink,linkcolor=myblue}
\usepackage{graphicx}
\usepackage{url,algorithm,algorithmic}
\usepackage{graphicx}
\usepackage{cite}
\usepackage[margin=1in]{geometry}
\usepackage{complexity} 
\usepackage[T1]{fontenc}
\usepackage{ae}
\usepackage{aecompl}

\newcommand{\vct}[1]{\bm{#1}}

\newtheorem{theorem}{Theorem}
\newtheorem{lemma}[theorem]{Lemma}
\newtheorem{proposition}[theorem]{Proposition}
\newtheorem{corollary}[theorem]{Corollary}
\newtheorem{remark}[theorem]{Remark}

\newcommand{\scalprod}[2]{\left\langle #1,#2 \right\rangle}

\begin{document}

\title{Recovery guarantees for exemplar-based clustering}

\author{
Abhinav Nellore\\
{The Johns Hopkins University}\\
{anellore@gmail.com}
\and
Rachel Ward \\
{The University of Texas at Austin}\\
{rward@math.utexas.edu}
}

\maketitle
\begin{abstract}
For a certain class of distributions, we prove that the linear programming relaxation of $k$-medoids clustering---a variant of $k$-means clustering where means are replaced by exemplars from within the dataset---distinguishes points drawn from nonoverlapping balls with high probability once the number of points drawn and the separation distance between any two balls are sufficiently large. Our results hold in the nontrivial regime where the separation distance is small enough that points drawn from different balls may be closer to each other than points drawn from the same ball; in this case, clustering by thresholding pairwise distances between points can fail. We also exhibit numerical evidence of high-probability recovery in a substantially more permissive regime.
\end{abstract}
\section{Introduction}
Consider a collection of points in Euclidean space that forms roughly isotropic clusters. The \textit{centroid} of a given cluster is found by averaging the position vectors of its points, while the \textit{medoid}, or exemplar, is the point {\it from within the collection} that best represents the cluster. To distinguish clusters, it is popular to pursue the $k$-means objective: partition the points into $k$ clusters such that the average squared distance between a point and its cluster centroid is minimized. This problem is in general NP-hard \cite{aloise2009np, dasgupta2009random}. Further, it has no obvious convex relaxation, which could recover the global optimum while admitting efficient solution; practical algorithms like Lloyd's\cite{lloyd1982least} and Hartigan-Wong \cite{hartigan1979algorithm} typically converge to local optima. $k$-medoids clustering\footnote{$k$-medoids clustering is sometimes called $k$-medians clustering in the literature.} is also in general NP-hard \cite{papadimitriou1981worst, megiddo1984complexity}, but it does admit a linear programming (LP) relaxation. The objective is to select $k$ points as medoids such that the average squared distance (or other measure of dissimilarity) between a point and its medoid is minimized. \textbf{This paper obtains guarantees for exact recovery of the unique globally optimal solution to the $k$-medoids integer program by its LP relaxation.} Commonly used algorithms that may only converge to local optima include partitioning around medoids (PAM) \cite{van2003new, kaufman2009finding} and affinity propagation \cite{frey2007clustering, givoni2009binary}.

\begin{figure}[!ht]\label{yalies}
\renewcommand{\thefootnote}{\thempfootnote}
\begin{center}
\includegraphics[width=5cm]{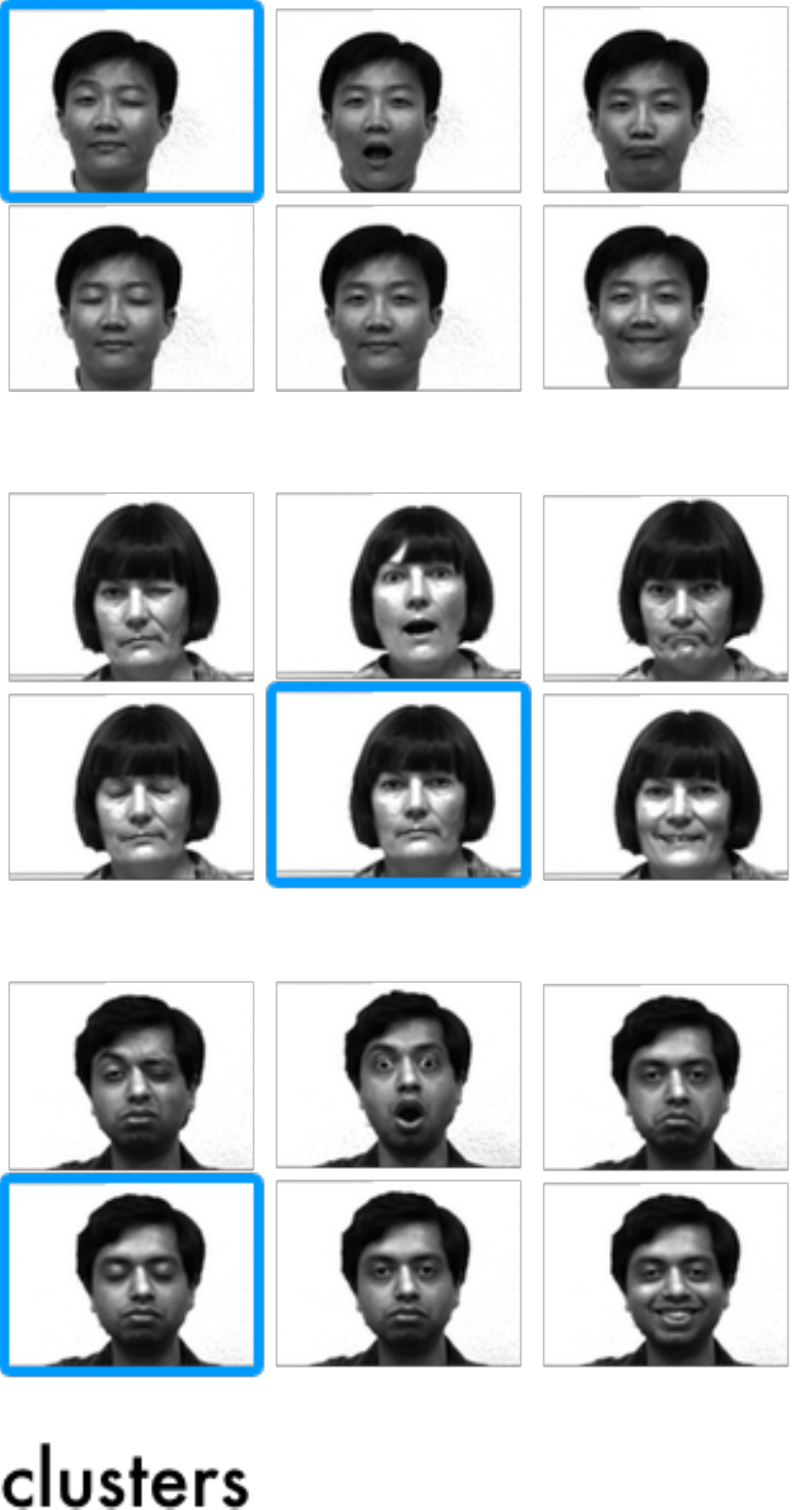} \quad \quad \quad \quad \includegraphics[width=5cm]{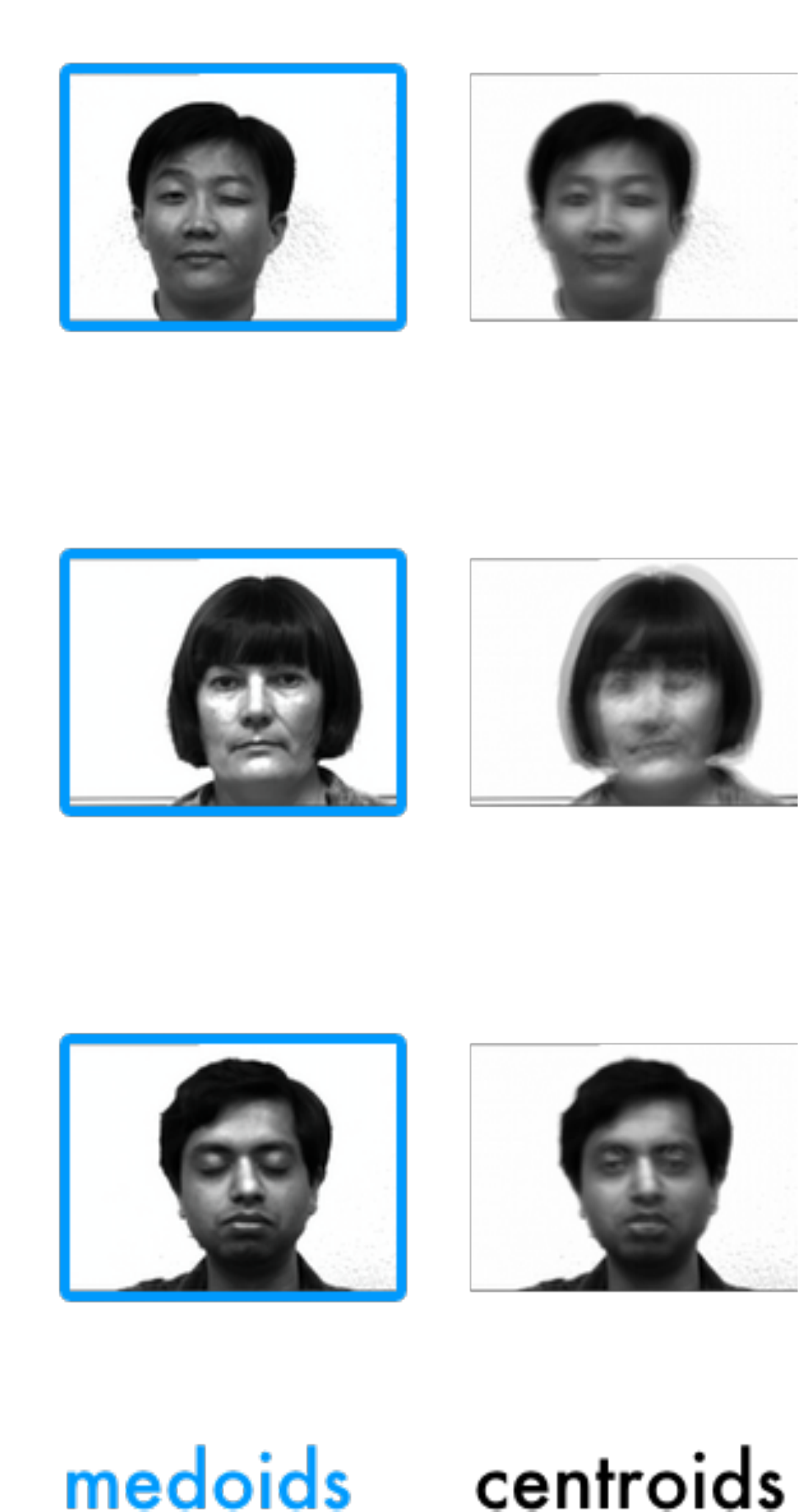} 
\end{center}
\caption{18 images of 3 faces $\times$ 6 facial expressions from the Yale Face Database were clustered using affinity propagation and Lloyd's algorithm. The medoids identified by affinity propagation (framed) are representative faces from the clusters, while the centroids found by Lloyd's algorithm are averaged faces.}
\end{figure}

To illustrate the difference between a centroid and a medoid, let us put faces to points. The Yale Face Database \cite{belhumeur1997eigenfaces} has grayscale images of several faces, each captured wearing a range of expressions---normal, happy, sad, sleepy, surprised, and winking. Suppose every point encodes an image from this database as the vector of its pixel values. Intuitively, facial expressions represent perturbations of a background composed of distinguishing image features; it is thus natural to expect that the faces cluster by individual rather than expression. Both Lloyd's algorithm and affinity propagation are shown to recover this partitioning in Figure \ref{yalies}, which also displays centroids and medoids of clusters.\footnote{$500$ randomly initialized repetitions of Lloyd's algorithm were run; the clustering that gave the smallest objective function value is shown. The package {\tt APCluster} \cite{bodenhofer2011apcluster} was used to perform affinity propagation.} The centroids are averaged faces, but the medoids are actual faces from the dataset. Indeed, applications of $k$-medoids clustering are numerous and diverse: besides finding representative faces from a gallery of images \cite{mezard2007computer}, it can group tumor samples by gene expression levels \cite{leone2007clustering} and pinpoint the influencers in a social network \cite{tang2009social}.

\subsection{Setup and principal result}
We formulate $k$-medoids clustering on a complete weighted undirected graph ${\cal G} = (V, E)$ with $N$ vertices, although recovery guarantees are proved for the case where vertices correspond to points in Euclidean space and {\bf each edge weight is the squared  $\ell_2$ distance between the points it connects.}\footnote{We use squared $\ell_2$ distances rather than unsquared $\ell_2$ distances only because we were able to derive stronger theoretical guarantees using squared $\ell_2$ distances.} Let characters in boldface (``$\mathbf{m}$'') refer to matrices/vectors and italicized counterparts with subscripts (``$m_{ij}$'') refer to matrix/vector elements. Denote as $w_{ij}$ the nonnegative weight of the edge connecting vertices $i$ and $j$, and note that $w_{ii}=0$ since ${\cal G}$ is simple. $k$-medoids clustering (\textsc{KMed}) finds the minimum-weight bipartite subgraph ${\cal G}' = ({\cal M}, V\backslash {\cal M}, E')$ of ${\cal G}$ such that $| {\cal M} |=k$ and every vertex in $V \backslash {\cal M}$ has unit degree. The vertices in ${\cal M}$ are the medoids. Expressed as a binary integer program, \textsc{KMed} is
\begin{align}
\min_{\vct{z} \in \mathbb{R}^{N \times N}} \, & \sum_{i=1}^N \sum_{j=1}^N w_{ij} z_{ij} \label{KMedObj} \\
\mbox{s.t. } &\sum_{j=1}^N z_{ij} = 1, \quad i \in [N] \label{KMed1} \\
& \sum_{j=1}^N z_{jj} \leq k \label{KMed2} \\
& z_{ij} \leq z_{jj}, \quad i, j \in [N] \label{KMed3} \\
& z_{ij} \in \{0, 1\} \label{KMed4}\,.
\end{align}
Above, $[N]$ means the set $\{1, \ldots, N\}$. When $z_{ij} = 1$, vertex $j \in {\cal M}$ serves as vertex $i$'s medoid; that is, among all edges between medoids and $i$, the edge between $j$ and $i$ has the smallest weight. Otherwise, $z_{ij} = 0$. A cluster is identified as a maximal set of vertices that share a given medoid.

Like many clustering programs, \textsc{KMed} is in general NP-hard and thus computationally intractable for a large $N$. Replacing the binary constraints (\ref{KMed4}) with nonnegativity constraints, we obtain the linear program relaxation \textsc{LinKMed}:
\begin{align}
\min_{ \vct{z} \in \mathbb{R}^{N \times N}} \, & \sum_{i=1}^N \sum_{j=1}^N w_{ij} z_{ij} \label{LinKMedObj} \\
\mbox{s.t. } &\sum_{j=1}^N z_{ij} = 1, \quad i \in [N] \label{LinKMed1}\\
& \sum_{i=1}^N z_{ii} \leq K \label{LinKMed2} \\
& z_{ij} \leq z_{jj}, \quad i, j \in [N] \label{LinKMed3} \\
& z_{ij} \geq 0 \label{LinKMed4}\,.
\end{align}

For a vector (point) $\vct{x} \in \mathbb{R}^d$, let $\| \vct{x} \|$ denote its $\ell_2$ norm. It is known that for any configuration of points in one-dimensional Euclidean space, the LP relaxation of $k$-medoids clustering invariably recovers $k$ clusters when unsquared distances are used to measure dissimilarities between points\cite{de1998k}. Therefore, we confine our attention to $d \geq 2$. The following is our main recovery result, and its proof is obtained in the third section.

\begin{theorem}\label{thm1}
Consider $k$ unit balls in $d$-dimensional Euclidean space (with $d \geq 2$) for which the centers of any two balls are separated by a distance of at least $3.75$. From each ball, draw $n$ points $\vct{x}_1, \vct{x}_2, \dots, \vct{x}_n$ as independent samples from a spherically symmetric distribution supported in the ball satisfying
\begin{equation}
\label{assumebeginner} 
{\bf Prob} ( \| \vct{x}_i \| \geq r ) \leq 1-r^2, \quad 0 \leq r \leq 1.
\end{equation}
Suppose that squared distances are used to measure dissimilarities between points, $w_{ij} = \| \vct{x}_i - \vct{x}_j \|^2$.  Then there exist values of $n$ and $k \geq 2$ for which the following statement holds: with probability exceeding $1-4k/n$, the optimal solution to $k$-medoids clustering (\textsc{KMed}) is unique and agrees with the unique optimal solution to (\textsc{LinKMed}), and assigns the points in each ball to their own cluster.
\end{theorem}
\begin{remark}
The uniform distribution satisfies \eqref{assumebeginner} in dimension $d=2$, but for $d > 2$, a distribution satisfying \eqref{assumebeginner} concentrates more probability mass towards the center of the ball.  This means that the recovery results of Theorem \ref{thm1} are stronger for smaller $d$. However, by applying a random projection, $n$ points in $d$ dimensions can be projected into $m = \mathcal{O}(\log n /\varepsilon^2)$ dimensions while preserving pairwise Euclidean distances up to a multiplicative factor $1 \pm \varepsilon$.  In this sense, clustering problems in high-dimensional Euclidean space can be reduced to problems in low-dimensional Euclidean space\cite{bzd10}.
\end{remark}

\begin{remark}
Once the centers of any two unit balls are separated by a distance of 4,  points from within the same ball are necessarily at closer distance than points from different balls.  For the k-medoid problem, cluster recovery guarantees in this regime are given in \cite{elhamifar2012finding}.  As far as the authors are aware, Theorem \ref{thm1} provides the first recovery guarantees for k-medoids beyond this regime.
\end{remark}


\subsection{Relevant works}

While the literature on clustering is extensive, three lines of inquiry are closely related to the results contained here.
\begin{itemize}
\item \textbf{Recovery guarantees for clustering by convex programming.} Our work is aligned in spirit with the tradition of the compressed sensing community, which has sought probabilistic recovery guarantees for convex relaxations of nonconvex problems. Reference \cite{ames2012guaranteed} presents such guarantees for the densest $k$-clique problem \cite{ames2010convex}: partition a complete weighted graph into $k$ disjoint cliques so that the sum of their average edge weights is minimized. Also notable are \cite{oymak2011finding, jalali2012clustering, chen2012clustering, jalali2011clustering}, which find recovery guarantees for correlation clustering \cite{bansal2004correlation} and variants. Correlation clustering outputs a partitioning of the vertices of a complete graph whose edges are labeled either ``$+$'' (agreement) or ``$-$'' (disagreement); the partitioning minimizes the number of agreements within clusters plus the number of disagreements between clusters.

In all papers mentioned in the previous paragraph, the probabilistic recovery guarantees apply to the stochastic block model (also known as the planted partition model) and generalizations. Consider a graph with $N$ vertices, initially without edges. Partition the vertices into $k$ clusters. The stochastic block model \cite{condon2001algorithms, holland1983stochastic} is a random model that draws each edge of the graph \textit{independently}: the probability of a ``$+$'' (respectively, ``$-$'') edge between two vertices in the same cluster is $p$ (respectively, $1-p$), and the probability of a ``$+$'' (respectively, ``$-$'') edge between two vertices in different clusters is $q < p$ (respectively, $1-q > 1-p$). Unfortunately, any model in which edge weights are drawn independently does not include graphs that represent points drawn independently in a metric space. For these graphs, the edge weights are \textit{interdependent} distances.

A recent paper \cite{soltanolkotabi2013robust} builds on \cite{elhamifar2012sparse, elhamifar2009sparse} to derive probabilistic recovery guarantees for subspace clustering: find the union of subspaces of $\mathbb{R}^d$ that lies closest to a set of points. This problem has only trivial overlap with ours; exemplars are ``zero-dimensional hyperplanes'' that lie close to clustered points, but there is only one zero-dimensional \textit{subspace} of $\mathbb{R}^d$---the origin. Reference \cite{elhamifar2012finding}, on the other hand, introduces a tractable convex program that does find medoids. This program can be recast as a dualized form of $k$-medoids clustering. However, the deterministic guarantee of \cite{elhamifar2012finding}:
\begin{enumerate}
\item applies only to the case where the clusters are recoverable by thresholding pairwise distances; that is, two points in the same cluster must be closer than two points in different clusters. Our probabilistic guarantees include a regime where such thresholding may fail.
\item specifies that a regularization parameter $\lambda$ in the objective function must be lower than some critical value for medoids to be recovered. $\lambda$ is essentially a dual variable associated with the $k$ of $k$-medoids, and it remains unchosen in the Karush-Kuhn-Tucker conditions used to derive the guarantee of \cite{elhamifar2012finding}. The number of medoids obtained is thus unspecified. By contrast, we guarantee recovery of a specific number of medoids.
\end{enumerate}

\item \textbf{Recovery guarantees for learning mixtures of Gaussians.} We derive recovery guarantees for a random model where points are drawn from isotropic distributions supported in nonoverlapping balls. This is a few steps removed from a Gaussian mixture model. Starting with the work of Dasgupta \cite{dasgupta1999learning}, several papers (a representative sample is \cite{sanjeev2001learning, vempala2004spectral, kannan2005spectral, achlioptas2005spectral, feldman2006pac, brubaker2009robust, belkin2009learning, chaudhuri2009learning}) already report probabilistic recovery guarantees for learning the parameters of Gaussian mixture models using algorithms unrelated to convex programming. Hard clusters can be found after obtaining the parameters by associating each point $i$ with the Gaussian whose contribution to the mixture model is largest at $i$. The questions here are different from our ours: under what conditions does a given polynomial-time algorithm---not a convex program, which admits many algorithmic solution techniques---recover the global optimum? How close are the parameters obtained to their true values? The progression of this line of research had been towards reducing the separation distances between the centers of the Gaussians in the guarantees; in fact, the separation distances can be zero if the covariance matrices of the Gaussians differ \cite{kalai2010efficiently, belkin2010polynomial}. Our results are not intended to compete with these guarantees. Rather, we seek to provide complementary insights into how often clusters of points in Euclidean space are recovered by LP.

\item \textbf{Approximation algorithms for $k$-medoids clustering and facility location.} As mentioned above, for any configuration of points in one-dimensional Euclidean space, the LP relaxation of $k$-medoids clustering exactly recovers medoids for dissimilarities that are unsquared distances \cite{de1998k}. In more than one dimension, nonintegral optima whose costs are lower than that of an optimal integral solution may be realized. There is a large literature on approximation algorithms for $k$-medoids clustering based on rounding the LP solution and other methods. This literature encompasses a family of related problems known as facility location. The only differences between the uncapacitated facility location problem (UFL) and $k$-medoids clustering are that 1) only certain points are allowed to be medoids, 2) there is no constraint on the number of clusters, and 3) there is a cost associated with choosing a given point as a medoid.

Constant-factor approximation algorithms have been obtained for metric flavors of UFL and $k$-medoids clustering, where the measures of distance between points used in the objective function must obey the triangle inequality. Reference \cite{shmoys1997approximation} obtains the first polynomial-time approximation algorithm for metric UFL; it comes within a factor of 3.16 of the optimum. Several subsequent works give algorithms that improve this approximation ratio \cite{guha1998greedy, korupolu1998analysis, charikar1999improved, mahdian2001greedy, jain2002new, chudak2003improved, jain2003greedy, sviridenko2006improved, mahdian2006approximation, byrka2007optimal}. It is established in \cite{guha1998greedy} that unless $\NP \subseteq \DTIME(n^{O(\log \log n)})$, the lower bounds on approximation ratios for metric UFL and metric $k$-medoids clustering are, respectively, $\alpha\approx1.463$ and $1+2/e\approx  1.736$. Here, $\alpha$ is the solution to $\alpha+1=\ln 2/\alpha$. In unpublished work, Sviridenko strengthens the complexity criterion for these lower bounds to $\P \neq \NP$.\footnote{See Theorem 4.13 of Vygen's notes \cite{vygen2005approximation} for a proof. We thank Shi Li for drawing our attention to this result.} The best known approximation ratios for metric UFL and metric $k$-medoids clustering are, respectively, $1.488$ \cite{li20121} and $1+\sqrt{3} + \epsilon\approx 2.732+\epsilon$ \cite{li2013approximating}. Before the 2012 paper \cite{li2013approximating}, only a $\left(3 + \epsilon\right)$-approximation algorithm had been available since 2001 \cite{arya2004local}. Because there is still a large gap between the current best approximation ratio for $k$-medoids clustering ($2.732$) and the theoretical limit ($1.736$), finding novel approximation algorithms remains an active area of research. Along related lines, a recent paper \cite{hajiaghayi2013constant} gives the first constant-factor approximation algorithm for a generalization of $k$-medoids clustering in which more than one medoid can be assigned to each point.

We emphasize that our results are recovery guarantees; instead of finding a novel rounding scheme for LP solutions, we give precise conditions for when solving an LP yields the $k$-medoids clustering. In addition, our proofs are for squared distances, which do not respect the triangle inequality.

\end{itemize}

\subsection{Organization}

The next section of this paper uses linear programming duality theory to derive sufficient conditions under which the optimal solution to the $k$-medoids integer program \textsc{KMed} coincides with the unique optimal solution of its linear programming relaxation \textsc{LinKMed}. The third section obtains probabilistic guarantees for exact recovery of an integer solution by the linear program, focusing on recovering clusters of points drawn from separated balls of equal radius. Numerical experiments demonstrating the efficacy of the linear programming approach for recovering clusters beyond our analytical results are reviewed in the fourth section. The final section discusses a few open questions, and an appendix contains one of our proofs.

\section{Recovery guarantees via dual certificates}

Let $M(i)$ be the index of vertex $i$'s medoid and $M(i, 2)$ be $\arg\!\min_{j \in {\cal M}, j \neq M(i)} w_{ij}$. For points in Euclidean space, $M(i, 2)$ is the index of the second-closest medoid to point $i$. For simplicity of presentation, take $w_{i, M(i,2)} = \infty$ when there is only one medoid. Denote as $S_i$ the set of points whose medoid is indexed by $i$. Let $(\, \cdot\, )_+$ refer to the positive part of the term enclosed in parentheses. Begin by writing a necessary and sufficient condition for a unique integral solution to \textsc{LinKMed}. 
\begin{proposition}\label{dualcert}
\textsc{LinKMed} has a unique optimal solution $\mathbf{x} = \mathbf{x}^\#$ that coincides with the optimal solution to \textsc{KMed} if and only if there exist some $u$ and $\vct{\lambda} \in \mathbb{R}^N$ such that
\begin{align}
u  >  \sum_{i=1}^N \left(\lambda_i - w_{ij} + w_{i, M(i)}\right)_+, \quad j \notin {\cal M} \nonumber \\
\sum_{i \in S_j} \lambda_i = u\,,\quad j \in {\cal M} \label{dual_cond}\\ 
0 \leq \lambda_i < w_{i, M(i, 2)} - w_{i, M(i)},\quad i \in [N]\,. \nonumber
\end{align}
\end{proposition}
Proposition \ref{dualcert} rewrites the Karush-Kuhn-Tucker (KKT) conditions corresponding to the linear program \textsc{LinKMed} in a convenient way; refer to the appendix for a derivation. Let $N_i$ be the number of points in the same cluster as point $i$. Choose $\lambda_i = u/N_{i}$ to obtain the following tractable sufficient condition for medoid recovery.
\begin{corollary}\label{forrecursion}
\textsc{LinKMed} has a unique optimal solution $\mathbf{z} = \mathbf{z}^\#$ that coincides with the optimal solution to \textsc{KMed} if there exists a $u \in \mathbb{R}$ such that
\begin{align}
\sum_{i=1}^N \left(T_{ij}\right)_+ < u < N_{\ell} \left(w_{\ell, M(\ell, 2)}-w_{\ell, M(\ell)}\right)\label{finalu}
\end{align}
for $j\notin {\cal M}$, $\ell \in [N]$, and
$$
T_{ij} = \frac{u}{N_i}+w_{i, M(i)}-w_{ij}\,.
$$
\end{corollary}
\begin{remark}
The choice of the KKT multipliers $\lambda_i$ made here is democratic: each cluster $S_j$ has a total of $u$ ``votes,'' which it distributes proportionally among the $\lambda_i$ for $i \in S_j$.
\end{remark}
Now consider the dual certificates contained in the following two corollaries.
\begin{corollary}\label{determcorollary}
If \textsc{KMed} has a unique optimal solution $\mathbf{z} = \mathbf{z}^\#$, \textsc{LinKMed} also has a unique optimal solution $\mathbf{z} = \mathbf{z}^\#$ when
\begin{equation}
\max_{i\in [N]}\max_{j \in S_{M(i)}} N_i\left(w_{ij} - w_{i, M(i)}\right) < \min_{i\in [N]}\min_{j \notin S_{M(i)}} N_i\left(w_{ij} - w_{i, M(i)}\right)\,.\label{generaldeterm}
\end{equation}
\end{corollary}
\begin{remark}\label{refdremark}
Choose points from within $k$ balls in $\mathbb{R}^d$, each of radius $r$, for which the centers of any two balls are separated by a distance of at least $R$. Measure $R$ in units of a ball's radius by setting $r=1$. Take $w_{ij} = (d_{ij})^p$, where $d_{ij}$ is the distance between points $i$ and $j$ and $p > 0$. The inequality \eqref{generaldeterm} is satisfied for
\begin{equation}\label{determresult}
R > 2 \left(1+\left(1+\frac{n_{max}}{n_{min}}\right)^{1/p}\right)
\end{equation}
by assigning the points chosen from each ball to their own cluster. Here, $n_{max}$ and $n_{min}$ are the maximum and minimum numbers of points drawn from any one of the balls, respectively. In the limit $p \to \infty$, \eqref{determresult} becomes $R > 4$.
\end{remark}
\begin{proof}
Impose both $T_{ij} \geq 0$ when points $i$ and $j$ are in the same cluster and $T_{ij} < 0$ when points $i$ and $j$ are in different clusters. Combined with \eqref{finalu}, the restrictions on $u$ are then
\begin{align}
\sum_{i \in S_{M(j)}}\left(w_{ij} - w_{i, M(i)}\right) > 0\,,\quad j \notin {\cal M} \label{weightcond}\\
N_i \left(w_{ij} - w_{i, M(i)}\right) \leq u < N_i \left(w_{i\ell} - w_{i, M(i)}\right)\,, \quad i \in [N]; j\in S_{M(i)}; \ell \notin S_{M(i)}\,. \label{bothcond}
\end{align}
Condition \eqref{weightcond} holds by definition of a medoid {\it unless the optimal solution to \textsc{KMed} itself is not unique.} In that event, it may be possible for a nonmedoid and a medoid in the same cluster to trade roles while maintaining solution optimality, making the LHS of \eqref{weightcond} vanish for some $j$. The phrasing of the corollary accommodates this edge case.
\end{proof}
\begin{remark}
The inequality \eqref{bothcond} requires $w_{ij} < w_{i\ell}$ for $i$ in the same cluster as $j$ but a different cluster from $\ell$. So any two points in the same cluster must be closer than any two points in different clusters. 
\end{remark}
Corollary \ref{determcorollary} does not illustrate the utility of LP for solving \textsc{KMed}. Given the conditions of a recovery guarantee, clustering could be performed without LP using some distance threshold $d_t$: place two points in the same cluster if the distance between them is smaller than $d_t$, and ensure two points are in different clusters if the distance between them is greater than $d_t$. In the separated balls model of Remark \ref{refdremark}, $R > 4$ guarantees that two points in the same ball are closer than two points in different balls. The next corollary is needed to obtain results for $R \leq 4$.
\begin{corollary}\label{usedcorollary}
Let
\begin{equation}
u = \max_{i\in [N]}\max_{j \in S_{M(i)}} N_i\left(w_{ij} - w_{i, M(i)}\right)\,.
\end{equation}
\textsc{LinKMed} has a unique optimal solution $\mathbf{z} = \mathbf{z}^\#$ that coincides with the optimal solution to \textsc{KMed} if
\begin{align}
u < N_i \left(w_{i, M(i,2)} - w_{i, M(i)}\right)\,,\quad i \in [N] \label{finalTijgreaterthan0}\\
\sum _{i \notin S_{M(j)}} \left(\frac{u}{N_i}+w_{i, M(i)} - w_{ij}\right)_+ < \sum_{i \in S_{M(j)}} \left(w_{ij} - w_{i, M(i)}\right)\,,\quad j \notin {\cal M}.\label{bigsuffcond}
\end{align}
\end{corollary}
\begin{proof}
Impose only $T_{ij} > 0$ when points $i$ and $j$ are in the same cluster so that together with \eqref{finalu}, the restrictions on $u$ are
\begin{align}
\sum _{i \notin S_{M(j)}} \left(\frac{u}{N_i}+w_{i, M(i)} - w_{ij}\right)_+ < \sum_{i \in S_{M(j)}} \left(w_{ij} - w_{i, M(i)}\right)\,,\quad j \notin {\cal M}\label{Tijgreaterthan0}\\
N_i\left(w_{ij} - w_{i, M(i)}\right) < u < N_{i} \left(w_{i, M(i, 2)}-w_{i, M(i)}\right)\,,\quad i \in [N]; j \notin {\cal M}\,.\label{needpos2}
\end{align}
To minimize the LHS of \eqref{Tijgreaterthan0}, choose $u$ so it approaches its lower bound..
\end{proof}
\begin{remark}
The inequality \eqref{needpos2} requires $w_{i, M(i,2)} > w_{ij}$ for $i$ and $j$ in the same cluster.
\end{remark}

Corollary \ref{determcorollary} imposes both extra upper bounds and extra lower bounds on $u$ in its proof. When two points in different clusters are closer than two points in the same cluster, $u$ cannot simultaneously satisfy these upper and lower bounds. To break this ``thresholding barrier,'' Corollary \ref{usedcorollary} imposes only extra lower bounds on $u$ and permits large $u$. Stronger recovery guarantees are obtained for large $u$ when medoids are sparsely distributed among the points. (Note that the optimal solution $\mathbf{z} = \mathbf{z}^\#$ is $k$-column sparse.) The next subsection obtains probabilistic guarantees using Corollary \ref{usedcorollary} for a variant of the separated balls model of Remark \ref{refdremark}.

\section{A recovery guarantee for separated balls}

The theorem stated in the introduction is proved in this section. Consider $k$ nonoverlapping $d$-dimensional unit balls in Euclidean space for which the centers of any two balls are separated by a distance of at least $R$. Take $w_{ij}$ to be the squared distance $d_{ij}^2 = \| \vct{x}_i - \vct{x}_j \|^2$ between points $\vct{x}_i$ and $\vct{x}_j$. Under a mild assumption about how points are drawn within each ball, the exact recovery guarantee of Remark \ref{refdremark} is extended in this subsection to the regime $R < 4$, where two points in the same cluster are not necessarily closer to each other than two points in different clusters. In particular, let the points in each ball correspond to independent samples of an isotropic distribution supported in the ball and which obeys
\begin{equation}\label{distassumption}
{\bf Prob}\left(\| \vct{x} \| \geq r \right) \leq 1-r^2, \quad 0 \leq r \leq 1\,,
\end{equation}
Above, $\vct{x} \in \mathbb{R}^d$ is the vector extending from the center of the ball to a given point, and $\| \vct{x} \|$ refers to the $\ell_2$ norm of $\vct{x}$. In $d=2$ dimensions, the assumption \eqref{distassumption} holds for the uniform distribution supported in the ball. For larger $d$, \eqref{distassumption} requires distributions that concentrate more probability mass closer to the ball's center.  For simplicity, we assume in the sequel that  the number $n$ of points drawn from each ball is equal. 
\noindent Let $\mathbb{E}$ denote an expectation and ${\bf Var}$ a variance. We state a preliminary lemma. 
\begin{lemma}\label{lem:bern}
Consider $\vct{x}_1, \dots, \vct{x}_n \in \mathbb{R}^d$ sampled independently from an isotropic distribution supported in a $d$-dimensional unit ball which satisfies
$$
\mathbf{Prob}(\| \vct{x}_i \| \geq  r) \leq 1 - r^2, \quad 0 \leq r \leq 1; i \in [n]\,.
$$
Use squared Euclidean distances to measure dissimilarities between points. Let $\vct{x}_{*}$ be the medoid of the set $\{\vct{x}_i\}$ and $\vct{x}_{min}=\operatornamewithlimits{argmin} \limits_{j} \| \vct{x}_j \|$. Assume that $d \geq 2$, $n \geq 3$, and $ 0 < \alpha \leq (3/2)\sqrt{(n-2)/2d}$. With probability exceeding $1 - ne^{-\alpha^2}$, all of the following statements are true.
\begin{enumerate}
\item $\sum_{j=1}^n (\| \vct{x}_j - \vct{x}_{\ell} \|^2 - \| \vct{x}_j - \vct{x}_{*} \|^2 ) \geq \sum_{j=1}^n (\| \vct{x}_j - \vct{x}_{\ell} \|^2 - \| \vct{x}_j - \vct{x}_{min} \|^2 ) \geq (n-2)\left(\|  \vct{x}_{\ell} \|^2 - \|  \vct{x}_{min} \|^2\right) -2\alpha \sqrt{n-2}\left(\|\vct{x}_{\ell} \| + \|\vct{x}_{min}\|\right)$ for all $\ell \in [n]$.
\item $\| \vct{x}_{min} \|  \leq  \alpha n^{-1/2}$.
\item $\| \vct{x}_{*} \| \leq 3 \alpha\left(n-2\right)^{-1/2}$.
\end{enumerate}
\end{lemma}
\begin{proof}
First prove statement 1. Note that for $\ell \in [n]$,
\begin{align}
 & -(n-2)\left(\|  {\vct{x}}_{\ell} \|^2 - \|  {\vct{x}}_{min} \|^2\right) + \sum_{j=1}^n \left(\| {\vct{x}}_j - {\vct{x}}_{\ell} \|^2 - \| {\vct{x}}_j - {\vct{x}}_{min} \|^2\right) \\
 =& -(n-2)\left(\|  {\vct{x}}_{\ell} \|^2 - \|  {\vct{x}}_{min} \|^2\right) + \sum_{j\neq \ell, j\neq min}\left(\| {\vct{x}}_j - {\vct{x}}_{\ell} \|^2 - \| {\vct{x}}_j - {\vct{x}}_{min} \|^2\right) \\
 =& -2\|\vct{x}_{\ell} - {\vct{x}}_{min} \| \sum_{j\neq \ell, j\neq min} \vct{y}_j \,. \label{identity}
\end{align}
Above, $\vct{y}_j = \scalprod{\vct{x}_j}{\left(\vct{x}_{\ell} - {\vct{x}}_{min}\right)/\|\vct{x}_{\ell} - {\vct{x}}_{min}\|}$. Since the  ${\vct{x}}_i$ are drawn from an isotropic distribution, the direction of the unit vector $\left(\vct{x}_{\ell} - {\vct{x}}_{min}\right)/\|\vct{x}_{\ell} - {\vct{x}}_{min}\|$ is independent of $\|\vct{x}_{\ell} - {\vct{x}}_{min}\|$ and drawn uniformly at random. It follows that the $\vct{y}_j$ for $j \neq min, j\neq \ell$ are i.i.d.\  zero-mean random variables despite how  ${\vct{x}}_{min}$ depends on the  ${\vct{x}}_j$. Indeed, for $j \neq min, j\neq \ell$,
$$
\mathbf{Var}(\vct{y}_j) = \mathbb{E}\left(|\scalprod{\vct{x}_j}{\left(\vct{x}_{\ell} - \vct{x}_{min}\right)/\|\vct{x}_{\ell} - \vct{x}_{min}\| }|^2\right) = \mathbb{E} \left(\| {\vct{x}}_j \|^2 \cos^2 \theta\right)= 1/(2d)\,,
$$
where the last equality is obtained by integrating in generalized spherical coordinates. Bernstein's inequality thus gives
\begin{align}\label{bernlock}
\mathbf{Prob} \left( \sum_{j\neq \ell,  j\neq min} \vct{y}_j  > \alpha \sqrt{\frac{2\left(n-2\right)}{d}} \right) &\leq e^{-\alpha^2}\,, \quad  0 < \alpha \leq \frac{3}{2}\sqrt{\frac{n-2}{2d}}.
\end{align}
So $\sum_{i=1}^n \vct{y}_i$ is bounded from above with high probability given \eqref{bernlock}. Further, $\|\vct{x}_{\ell}- {\vct{x}}_{min}\| \leq \|\vct{x}_{\ell} \| + \| {\vct{x}}_{min}\|$ from the triangle inequality. These facts together with \eqref{identity} imply that {\it for a given $\ell \neq min$},
\begin{align}\label{forconsolid}
&\sum_{j=1}^n \left(\| {\vct{x}}_j - {\vct{x}}_{\ell} \|^2 - \| {\vct{x}}_j - {\vct{x}}_{*} \|^2\right) \\ \geq &\sum_{j=1}^n \left(\| {\vct{x}}_j - {\vct{x}}_{\ell} \|^2 - \| {\vct{x}}_j - {\vct{x}}_{min} \|^2\right) \mbox{  (by definition of a medoid)}\\
\geq &(n-2)\left(\|  {\vct{x}}_{\ell} \|^2 - \|  {\vct{x}}_{min} \|^2\right) -2\alpha \sqrt{\frac{2\left(n-2\right)}{d}}\left(\|\vct{x}_{\ell} \| + \|\vct{x}_{min}\|\right)\\
\geq &(n-2)\left(\|  {\vct{x}}_{\ell} \|^2 - \|  {\vct{x}}_{min} \|^2\right) -2\alpha \sqrt{n-2}\left(\|\vct{x}_{\ell}\| + \|\vct{x}_{min}\|\right)\mbox{  (for $d \geq 2$)}
\end{align}
 with probability exceeding $1 - e^{-\alpha^2}$. For $\ell =min$, the inequalities above clearly hold with unit probability. Take a union bound over the other $\ell \in [n]$ to obtain that statement 1 holds with probability exceeding $1 - \left(n-1\right) e^{-\alpha^2}$ (for valid $\alpha$ as specified in \eqref{bernlock}). 

Now observe that 
$$
\mathbf{Prob}(\| {\vct{x}}_{min} \| > \alpha n^{-1/2}) = \big( 1 - \alpha^{2} n^{-1}  \big)^n < e^{- \alpha^{2}}\,.
$$
It follows that statement 2 holds with probability exceeding $1 - e^{-\alpha^2}$. Moreover, statements 1 and 2 together hold with probability exceeding $1 - n e^{-\alpha^2}$. Condition on them, and prove statement 3 by contradiction: suppose that $\| {\vct{x}}_{*} \|$ exceeds $3 \alpha (n-2)^{-1/2}$. Then because $\| {\vct{x}}_{min} \| \leq \alpha n^{-1/2} < \alpha (n-2)^{-1/2}$,
\begin{equation}
\left(n-2\right) \left(\| {\vct{x}}_{*} \|^2 - \| {\vct{x}}_{min} \|^2\right)- 2 \alpha \sqrt{n-2} \left( \| {\vct{x}}_{*} \| +  \| {\vct{x}}_{min} \| \right) > 8 \alpha^2  - 8\alpha^2 = 0\,.
\end{equation}
But from statement 1 for $\ell=min$, this implies that
\begin{equation}
\sum_i \| {\vct{x}}_j - {\vct{x}}_{*}  \|^2 > \| {\vct{x}}_{j} - {\vct{x}}_{min}  \|^2\,,
\end{equation}
which violates the assumption that  ${\vct{x}}_{*}$ is a medoid. So all three statements hold with probability exceeding $1 - n e^{-\alpha^2}$, which is the content of the lemma.
\end{proof}

\noindent We now write the main result of this section.  
\begin{theorem}{\bf (Restatement of Theorem \ref{thm1}.)}\label{main}
Consider $k$ unit balls in $d$-dimensional Euclidean space (with $d \geq 2$) for which the centers of any two balls are separated by a distance of $3.75+\varepsilon$, $\varepsilon \geq 0$.  From each ball, draw $n$ points $\vct{x}_1, \vct{x}_2, \dots, \vct{x}_n$ as independent samples from an isotropic distribution supported in the ball which satisfies
\begin{equation}
\label{assumebeginner2} 
{\bf Prob} ( \| \vct{x}_i \| \geq r ) \leq 1-r^2, \quad 0 \leq r \leq 1.
\end{equation}
Suppose that squared distances are used to measure dissimilarities between points, $w_{ij} = \| \vct{x}_i - \vct{x}_j \|^2$.    For each $\varepsilon \geq 0$, there exist values of $n$ and $k$ for which the following statement holds: with probability exceeding $1-4k/n$, the unique optimal solution to each of $k$-medoids clustering and its linear programming relaxation assigns the points in each ball to their own cluster.
\end{theorem}
\begin{remark}\label{tableremark}
A table of valid combinations of $\varepsilon$, $n$, and $k$ (for $d \leq \frac{9}{8}\frac{n-2}{\log n}$) follows.
\begin{center}
  \begin{tabular}{| c | c | c |}
    \hline
$\varepsilon$ & $n$ & $k$ \\ \hline
 $\geq 0$ & $\geq 10^6$ & $2$ \\ \hline
 $\geq 0.05$ & $\geq 10^7$ & $\leq 3$ \\ \hline
$\geq 0.15$ & $\geq 10^4$ & $2$ \\ \hline
 $\geq 0.15$ & $\geq 10^7$ & $\leq 10$ \\ \hline
  \end{tabular}
\end{center}
More such combinations may be obtained by satisfying the inequalities \eqref{secondRconstraint}, \eqref{thirdRconstraint}, and \eqref{theend} in the proof below.  
\end{remark}

\begin{proof}
Condition on the $k$ events that the three statements of Lemma \ref{lem:bern} hold for each ball. Choose $\alpha=\sqrt{2\log n}$ so that the probability these events occur together exceeds $1 - k/n$. Because $\alpha$ is bounded by Lemma \ref{lem:bern}, this requires 
\begin{equation}\label{firstnconstraint}
d \leq \frac{9}{8}\left(\frac{n-2}{\log n}\right)\,.
\end{equation}
Now simplify the sufficient condition of Corollary \ref{usedcorollary} with $w_{ij} = d_{ij}^2$ and every $N_i = n$. Let
$$
\rho = 3\sqrt{\frac{2 \log n}{n-2}}
$$
be the maximum distance of a medoid to the center of its respective ball from statement 3 of Lemma \ref{lem:bern}.
Note that
$$
u=\max_{i \in [N]} \max_{j \in S_{M(i)}} n (d^2_{ij} - d^2_{i,M(i)}) \leq n\left(4 - (1-\rho)^2\right)\,,
$$
where the upper bound is surmised by considering point $M(i)$ collinear with and between points $i$ and $j$, both on one ball's boundary. So take $u = n\left(4-\left(1-\rho\right)^2\right)$ to narrow the sufficient condition of Corollary \ref{usedcorollary}. Also note the requirement \eqref{finalTijgreaterthan0}:
\begin{equation}\label{anotherrequire}
u < \min_{i \in [N]}  \left(d_{i, M(i,2)}^2 - d_{i, M(i)}^2\right)\,.
\end{equation}
Obtain a lower bound of $\left(R-(1+\rho)\right)^2-(1+\rho)^2$ on the RHS by considering a point $i$ on the boundary of one ball collinear with points $M(i)$ and $M(i, 2)$. Impose
\begin{equation}\label{firstRconstraint}
(R-(1+\rho))^2-(1+\rho)^2 > 4 - (1-\rho)^2
\end{equation}
to ensure the RHS exceeds the LHS. This is equivalent to
\begin{equation}\label{secondRconstraint}
R > 1+\rho+2\sqrt{1+\rho}\,.
\end{equation}

Given the stipulations of the previous paragraph and \eqref{bigsuffcond} of Corollary \ref{usedcorollary}, the following holds: each of $k$-medoids clustering and its LP relaxation has a unique optimal solution that assigns the points in each ball to their own cluster if for $j \notin \mathcal{M}$,
\begin{equation}\label{sublime}
\sum_{i \notin S_{M(j)}}\big(\| \vct{x}_i - {\vct{x}}_{M(i)}\|^2 - \| \vct{x}_i - {\vct{x}}_{j}\|^2 + 4-(1-\rho)^2\big)_+ < \sum_{i \in S_{M(j)}}\left(\| \vct{x}_i - {\vct{x}}_{j}\|^2 - \| \vct{x}_i - {\vct{x}}_{M(i)}\|^2\right)\,.
\end{equation}
Denote as $C(j)$ the center of the ball associated with point $j$.\footnote{This becomes a slight abuse of notation because $C(j)$ is not an index of any point drawn, but all its usages contained here should be clear.} Find conditions under which \eqref{sublime} holds by treating two complementary cases of the  ${\vct{x}}_j$ separately:
\begin{enumerate}
\item $d_{j, C(j)} \leq R - 1 - 2\sqrt{1+\rho}$. Then for $i \notin S_{M(j)}$, consider point $i$ collinear with and between points $M(i)$ and $j$ to obtain
\begin{equation}\label{detupperbnd}
\begin{aligned}
&\| \vct{x}_i - {\vct{x}}_{M(i)}\|^2 - \| \vct{x}_i - {\vct{x}}_{j}\|^2 + 4-(1-\rho)^2 \\
\leq &\left(1+\rho\right)^2 - \left(R-2 + 1 - (R - 1 - 2\sqrt{1+\rho})\right)^2 + 4-(1-\rho)^2\\
= &\,0.
\end{aligned}
\end{equation}
It follows that the LHS of \eqref{sublime} has an upper bound that vanishes. Moreover, the RHS of \eqref{sublime} must be positive by definition of a (unique) medoid. So \eqref{sublime} holds with unit probability.
\item $R - 1 - 2\sqrt{1+\rho} < d_{j, C(j)} \leq 1$. First, bound the number $n_{outer}$ of points in a given cluster for which the inequalities spanning the previous sentence hold. From the distribution \eqref{assumebeginner},
\begin{equation}
{\bf Prob}\left(R - 1 - 2\sqrt{1+\rho} < d_{j, C(j)} \leq 1\right) = \left(1 - \left(R-1-2\sqrt{1+\rho}\right)^2\right)_+\,.
\end{equation}
Hoeffding's inequality thus gives
\begin{equation}
{\bf Prob}\left(n_{outer} \geq n\left(1 - \left(R-1-2\sqrt{1+\rho}\right)^2\right)_+ +n\tau\right) \leq e^{-2n\tau^2}\,.
\end{equation}
Take $\tau = \sqrt{\dfrac{\log n}{2n}}$ to obtain
\begin{equation}\label{bndonnouter}
{\bf Prob} \left(n_{outer} < n\left(1 - \left(R-1-2\sqrt{1+\rho}\right)^2\right)_+ +\sqrt{(n/2)\log n}\right) > 1 - \frac{1}{n}\,.
\end{equation}
Condition on the event captured in the equality above holding for every cluster. This occurs with probability exceeding $1 - k/n$. Next, observe that for $i \notin S_{M(j)}$, considering (as for \eqref{detupperbnd}) point $i$ collinear with and between points $M(i)$ and $j$ gives the deterministic bound
\begin{equation}
\big(\| \vct{x}_i - {\vct{x}}_{M(i)}\|^2 - \| \vct{x}_i - {\vct{x}}_{j}\|^2 + 4-(1-\rho)^2\big)_+ \leq \left(1+\rho\right)^2 - \left(R-2\right)^2+4-\left(1-\rho\right)^2\,.
\end{equation}
Combine this with the bound on $n_{outer}$ from \eqref{bndonnouter} to find that for all $j \notin \mathcal{M}$, the LHS of \eqref{sublime} obeys
\begin{equation}\label{concentrate0}
\begin{aligned}
\sum_{i \notin S_{M(j)}}\big(d_{i, M(i)}^2 - d_{ij}^2 + 4-&(1-\rho)^2\big)_+ < \left(k-1\right)\left(\left(1+\rho\right)^2 - \left(R-2\right)^2+4-\left(1-\rho\right)^2\right)\\
&\times\left(n\left(1 - \left(R-1-2\sqrt{1+\rho}\right)^2\right)_+ +\sqrt{(n/2)\log n}\right)\,.
\end{aligned}
\end{equation}
Statement 1 of Lemma \ref{lem:bern} bounds from below the RHS of \eqref{sublime}:
\begin{equation}
\label{concentrate1}
\begin{aligned}
&\sum_{i \in S_{M(j)}} \big(d_{ij}^2 - d_{i, M(i)}^2 \big) \\
\geq &\left(n-2\right) \| {\vct{x}}_j \|^2 - 2  \sqrt{2\left(n-2\right)\log(n)} \| {\vct{x}}_j \|  - 2\sqrt{\frac{n-2}{n}}\left(2+\sqrt{\frac{n-2}{n}}\right) \log(n),
\end{aligned}
\end{equation}
where  ${\vct{x}}_j$ extends from the center of the ball corresponding to $S_{M(j)}$. The expression on the RHS has a minimum at $\| \vct{x}_j\| = \rho/3$. Because $1 \geq d_{j, C(j)} > R - 1 - 2\sqrt{1+\rho}$, provided
\begin{equation}\label{thirdRconstraint}
R - 1 - 2\sqrt{1+\rho} > \rho/3,
\end{equation}
a lower bound on the LHS of \eqref{concentrate1} is given by
\begin{equation}
\label{concentrate2}
\begin{aligned}
&\sum_{i \in S_{M(j)}} \big(d_{ij}^2 - d_{i, M(i)}^2 \big) \\
\geq &\left(n-2\right)\left(\min\left\{R - 1 - 2\sqrt{1+\rho}, 1\right\}\right)^2 - 2  \sqrt{2\left(n-2\right)\log(n)} \min\left\{R - 1 - 2\sqrt{1+\rho}, 1\right\} \\
&- 2\sqrt{\frac{n-2}{n}}\left(2+\sqrt{\frac{n-2}{n}}\right) \log(n)\,.
\end{aligned}
\end{equation}
Combining \eqref{concentrate0} and \eqref{concentrate2} provides a sufficient condition for \eqref{sublime}:
\begin{equation}\label{theend}
\begin{aligned}
&\left(n\left(1 - \left(R-1-2\sqrt{1+\rho}\right)^2\right)_+ +\sqrt{(n/2)\log n}\right)\\
&\times \left(K-1\right)\left(\left(1+\rho\right)^2 - \left(R-2\right)^2+4-\left(1-\rho\right)^2\right)\\
\leq &\left(n-2\right)\left(\min\left\{R - 1 - 2\sqrt{1+\rho}, 1\right\}\right)^2 \\
&- 2  \sqrt{2\left(n-2\right)\log(n)} \min\left\{R - 1 - 2\sqrt{1+\rho}, 1\right\} \\
&- 2\sqrt{\frac{n-2}{n}}\left(2+\sqrt{\frac{n-2}{n}}\right) \log(n)\,.
\end{aligned}
\end{equation}
 It is easily verified numerically that this inequality is satisfied when $R \geq 3.75$ for $n \geq 10^6$---as are the other bounds \eqref{secondRconstraint} and \eqref{thirdRconstraint} on $R$. Further, for any dimension $d$, there exists some finite $n$ large enough such that \eqref{firstnconstraint} is satisfied. Similar checks may be performed to obtain other valid combinations of the parameters; more such combinations are contained in Remark \ref{tableremark}.
\end{enumerate}

In the proof above, the sole events conditioned on are that Lemma \ref{lem:bern} holds for each cluster and that the Hoeffding inequality \eqref{bndonnouter} for $n_{outer}$ holds for each cluster. As recorded in the theorem, the probability all of these events occur exceeds $1-2k/n$. All components of the theorem are now proved.
\end{proof}

\section{Simulations}

Consider $k$ nonoverlapping $d$-dimensional unit balls in $\mathbb{R}^d$ for which the separation distance between the centers of any two balls is \textit{exactly} $R$. Consider the two cases that follow, referenced later as Case 1 and Case 2.
\begin{enumerate}
\item Each ball is the support of a uniform distribution.
\item Each ball is the support of a distribution that satisfies 
\begin{equation}\label{distsim}
{\bf Prob} ( \| \vct{x} \| \geq r ) = 1-r^2, \quad 0 \leq r \leq 1,
\end{equation}
where $r$ is the Euclidean distance from the center of the ball, and $\vct{x}$ is some vector in $\mathbb{R}^d$. For $d=2$, this is a uniform distribution. Equation \eqref{distsim} saturates the inequality \eqref{distassumption}, which is the distributional assumption of our probabilistic recovery guarantees.
\end{enumerate}
Given one of these cases, sample each of the $k$ distributions $n$ times so that $n$ points are drawn from each ball. Solve \textsc{LinKMed} for this configuration of points and record when 
\begin{enumerate}
\item a solution to \textsc{KMed} is recovered. (Call this ``cluster recovery.'')
\item a recovered solution to \textsc{KMed} places points drawn from distinct balls in distinct clusters, the situation for which our recovery guarantees apply. (Call this ``ball recovery,'' a sufficient condition for cluster recovery.)
\end{enumerate}
Examples of ball recoveries and cluster recoveries that are not ball recoveries are displayed in Figure \ref{failgrid} for $k=2, 3$.

\begin{figure}[!ht]
\begin{center}
\includegraphics[width=.75\paperwidth]{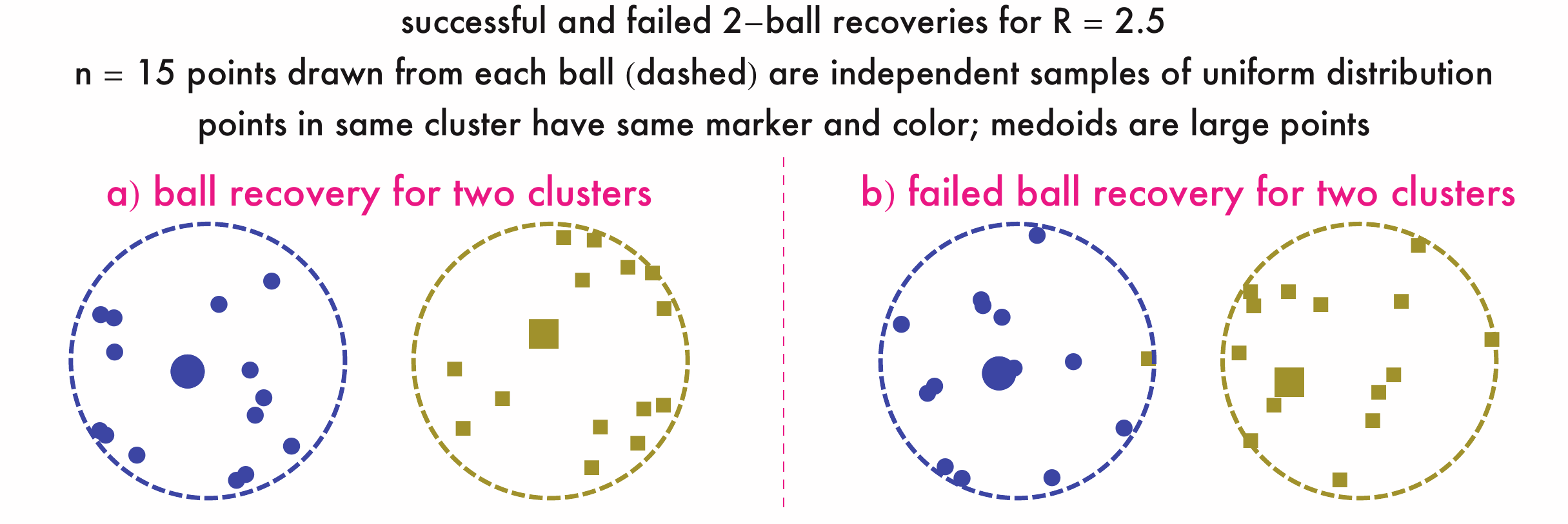}\\
\vspace{-.08in}
\includegraphics[width=.75\paperwidth]{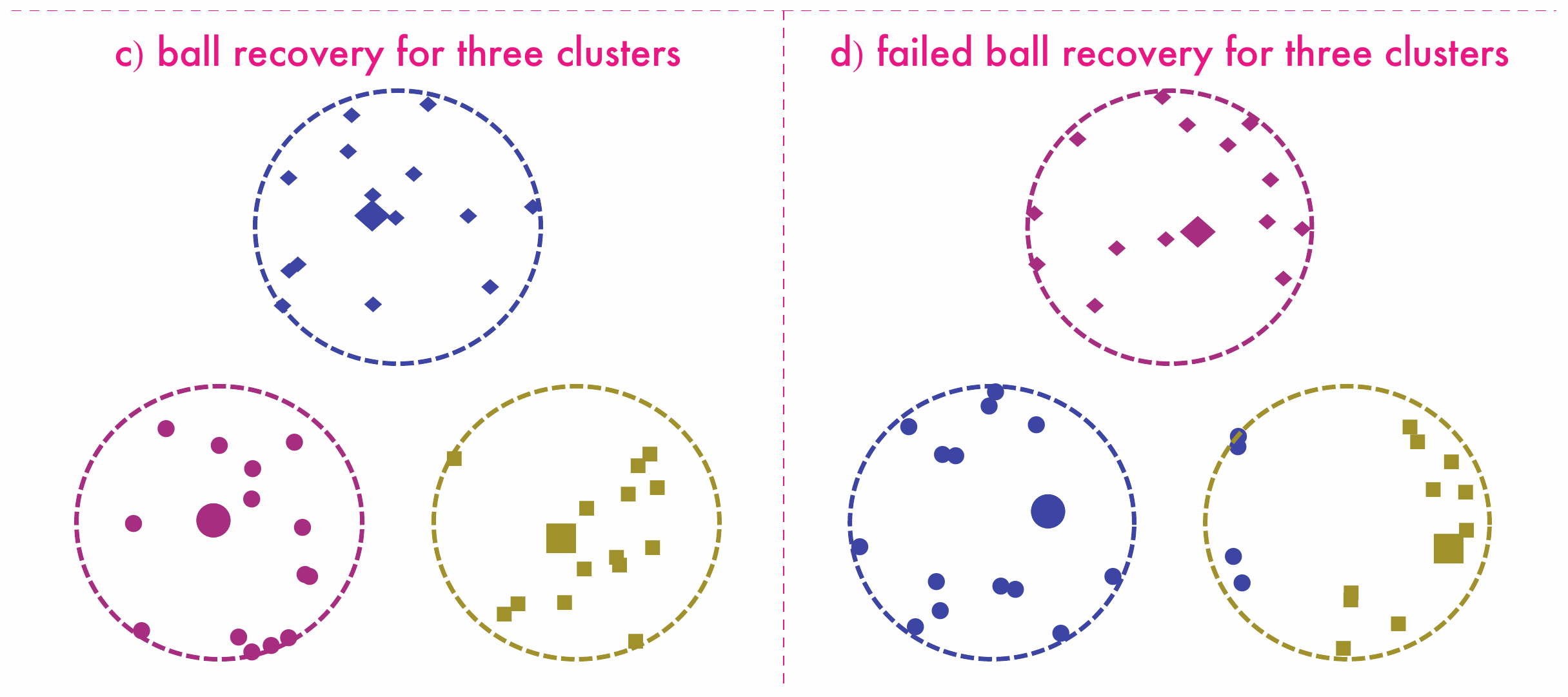} 
\end{center}
\vspace{-.2in}\caption{In the failed ball recovery for two clusters b), just one point in the left ball is not placed in the same cluster as the other points in the ball. In the failed ball recovery for three clusters d), four points in the bottom right ball are not placed in the same cluster as the other points in the ball.}
\label{failgrid}
\end{figure}
\clearpage

We performed $1000$ such simulations using MATLAB in conjunction with Gurobi Optimizer 5.5's barrier method implementation for every combination of the choices in the table below.\\\\
\begin{centering}
\begin{tabular}{|r|l|}
  \hline
  $n$ & 5, 10, 15, 20, 25, 30 \\
  $k$ & 2, 3 \\
  $R$ & 2, 2.2, 2.4, 2.6, 2.8, 3, 3.2, 3.4, 3.6, 3.8, 4, 4.2, 4.4, 4.6, 4.8, 5 \\ 
  $d$ & 2, 3, 4, 10 \\
  Cases & 1, 2 \\
  \hline
\end{tabular}\\
\end{centering}
\vspace{.13in}
Remarkably, cluster recovery failed no more than 12 (8) times out of 1000 across all sets of 1000 simulations for Case 1 (2). It therefore appears that high-probability cluster recovery is always realized when drawing samples from the distributions we consider. However, since the KKT conditions require some assumption about how points cluster, general cluster recovery guarantees are difficult to prove. In the previous section, we obtain guarantees assuming the points cluster into the balls from which they are drawn. The ball recovery results of our simulations for Cases 1 and 2 are depicted in, respectively, Figures \ref{unigrid} and \ref{r2grid}. Note that the vertical axis of each plot measures the number of \textit{failed} ball recoveries. We conclude this section with the following observations.
\begin{itemize}
\item On the whole, Case 2 yields more ball recoveries than Case 1. This is not unexpected: with the exception of $d=2$, Case 2 concentrates more probability mass towards the centers of the balls than does Case 1, typically making the points drawn from each ball cluster more tightly. For $d=2$, the plots in both Figures \ref{unigrid} and \ref{r2grid} correspond to draws from uniform distributions supported in the balls; they are repetitions and thus look essentially the same.
\item In general, as the number $n$ of points drawn from each ball is increased, the number of ball recoveries increases for fixed $d$, $k$, and $R$. This is again not unexpected: if fewer points are drawn, clustering is more susceptible to outliers that can prevent ball recovery.
\item As $R$ increases, the number of ball recoveries increases for fixed $d$, $k$, and $n$ because points drawn from different balls tend to get further apart. For $d=2$, high-probability ball recovery appears to be guaranteed for $R$ greater than somewhere between $2$ and $3$ even for the small values of $n$ considered here. This is considerably better than the guarantee of Theorem \ref{main}: it holds for $d=2$ and $R = 3.75$ only if $n$ is at least $10^6$, as shown toward the end of its proof.
\item For $n$, $R$, and $d$ fixed, there are more ball recoveries for two balls than there are for three balls. This suggests that as $k$ increases, the probability of recovery decreases, which is consistent with intuition from Theorem \ref{main}.
\item For $n$, $R$, and $k$ fixed, as $d$ increases, the number of ball recoveries increases, even for the uniform distributions of Case 1. There is thus substantial room for improving our recovery guarantees, which require concentrating more probability mass towards the centers of the balls as $d$ increases.
\end{itemize}
\clearpage

\begin{figure}[H]
\noindent\makebox[\textwidth]{\includegraphics[width=.75\paperwidth]{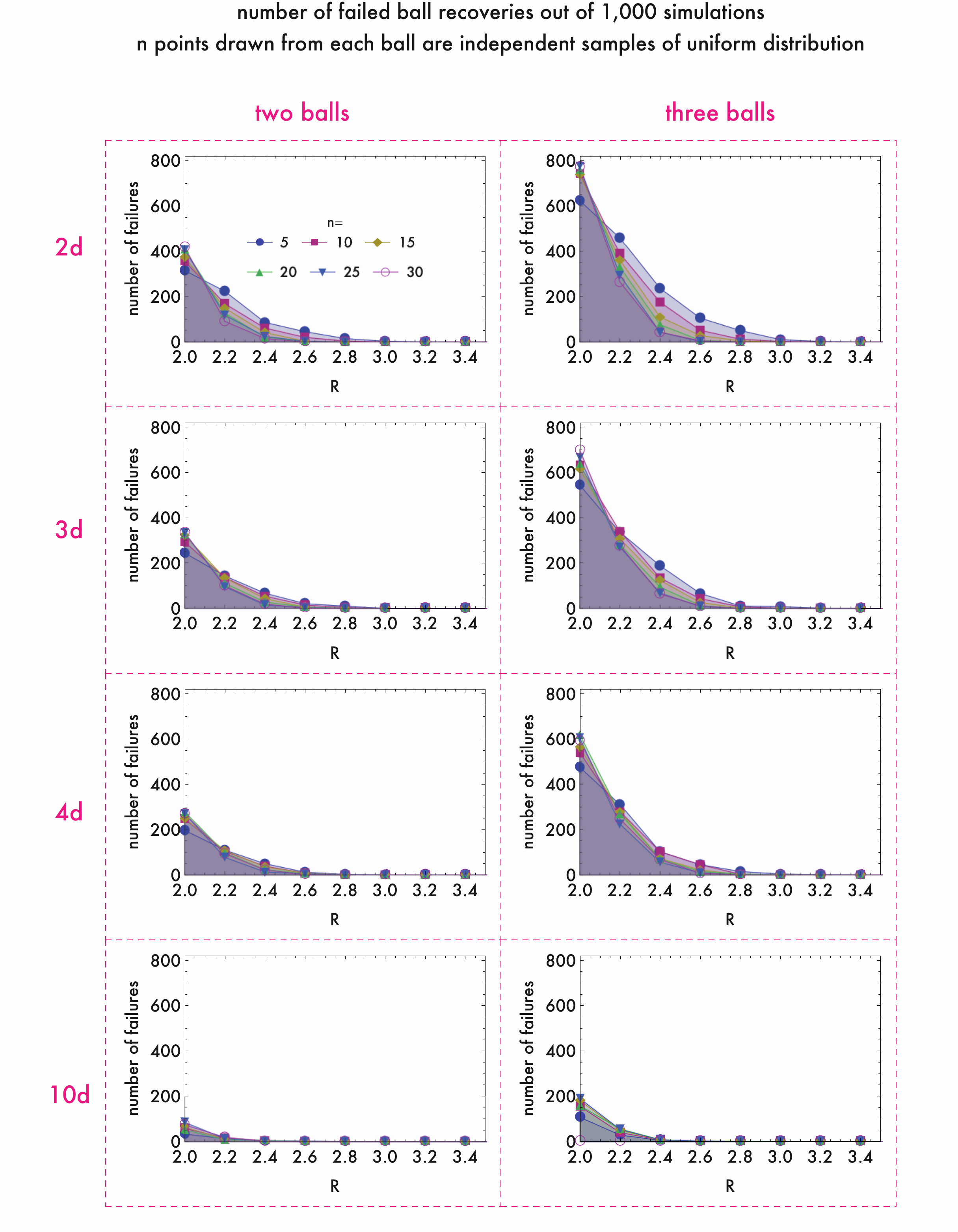}}
\vspace{-.4in}\caption{Exact recovery appears to be guaranteed with high probability for uniform distributions for values of $n$ in the double digits and values of $R$ below $3$. This is substantially better than Theorem \ref{main} suggests.}
\label{unigrid}
\end{figure}

\begin{figure}[H]
\noindent\makebox[\textwidth]{\includegraphics[width=.75\paperwidth]{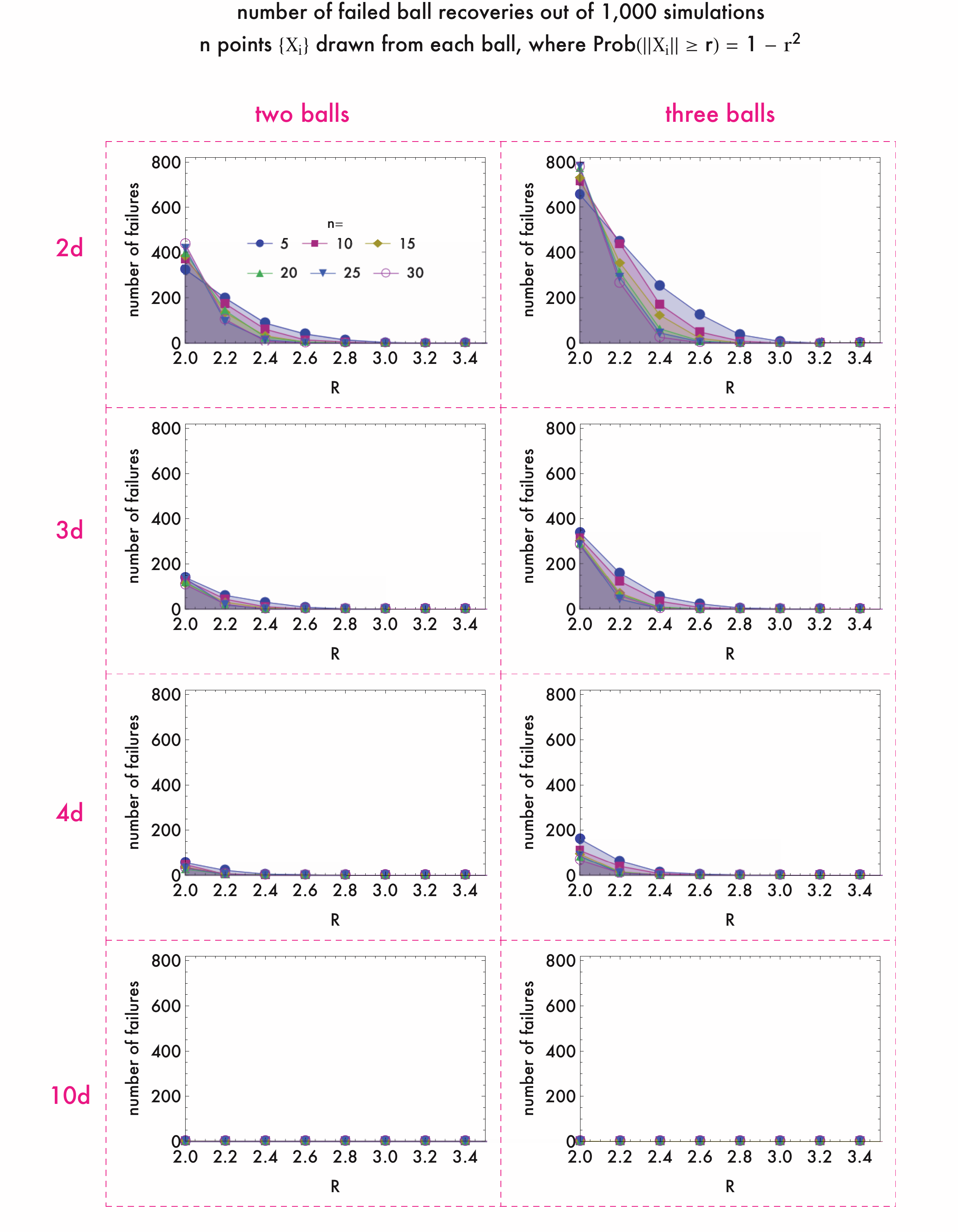}}
\vspace{-.4in}\caption{The $d=2$ plots here and in Figure \ref{unigrid} are repetitions of the same set of simulations. As $d$ increases, the number of exact recoveries increases faster than it does for uniform distributions.}
\label{r2grid}
\end{figure}
 
\section{Concluding remarks}

We proved that with high probability, the $k$-medoids clustering problem and its LP relaxation share a unique globally optimal solution in a nontrivial regime, where two points in the same cluster may be further apart than two points in different clusters. However, our theoretical guarantees are preliminary; they fall far short of explaining the success of LP  in distinguishing points drawn from different balls at small separation distance and with few points in each ball. More generally, in simulations we did not present here, the $k$-medoids LP relaxation appeared to recover integer solutions for very extreme configurations of points---in the presence of extreme outliers as well as for nonisotropic clusters with vastly different numbers of points. We thus conclude with a few open questions that interest us.
\begin{itemize}
\item How do recovery guarantees change for different choices of the dissimilarities between points---for example, for Euclidean distances rather than for the squared Euclidean distances used here? What about for Gaussian and exponential kernels?
\item Can exact recovery be used to better characterize outliers?
\item Is it possible to obtain cluster recovery guarantees instead of just ball recovery guarantees? (``Cluster recovery'' and ``ball recovery'' are defined right after \eqref{distsim}.)
\end{itemize} 
\section*{Acknowledgements}
We thank Shi Li and Chris White for helpful suggestions. We are extremely grateful to Sujay Sanghavi for offering his expertise on clustering and for pointing us in the right directions as we navigated the literature. A.N. is especially grateful to Jun Song for his constructive suggestions and for general support during the preparation of this work. R.W. was supported in part by a Research Fellowship from the Alfred P. Sloan Foundation, an ONR Grant N00014-12-1-0743, an NSF CAREER Award, and an AFOSR Young Investigator Program Award. A.N. was supported by Jun Song's grant R01CA163336 from the National Institutes of Health.

\clearpage
\newpage
\section*{Appendix: derivation of Proposition \ref{dualcert}}

\begin{proposition}\label{dualcertrestatement}
{\bf (Restatement of Proposition \ref{dualcert}.)} \textsc{LinKMed} has a unique solution $\mathbf{z} = \mathbf{z}^\#$ that coincides with the solution to \textsc{KMed} if and only if there exist some $u$ and $\vct{\lambda} \in \mathbb{R}^N$ such that
\begin{align}
u  >  \sum_{i=1}^N \left(\lambda_i - w_{ij} + w_{i, M(i)}\right)_+, \quad j \notin {\cal M} \nonumber \\
\sum_{i \in S_j} \lambda_i = u\,,\quad j \in {\cal M} \label{dual_cond}\\ 
0 \leq \lambda_i < w_{i, M(i, 2)} - w_{i, M(i)},\quad i \in [N]\,. \nonumber
\end{align}
\end{proposition}

\begin{proof}
Suppose the solution to \textsc{KMed} ${\bf z}= {\bf z}^\#$ is known. Let $\Omega$ be the index set of nonzero entries of ${\bf z}^\#$, and let $\Omega^c$ be its complement. For some matrix $\mathbf{m}$, denote as $\mathbf{m}_{\Omega^c}$ the vector of $N \times (N-1)$ variables $m_{ij}$ for which $(i, j) \in \Omega_c$. Eliminating the $z_{i, M(i)}$ from \textsc{LinKMed} using the constraints (\ref{LinKMed1}) yields the following equivalent program:
\begin{align}
\min_{{\bf z_{\Omega^c}} \in \mathbb{R}^{N \times N}} \, & \sum_{i=1}^N \sum_{j\neq M(i)} p_{ij} z_{ij} \label{LinKMedObj-2} \\
\mbox{s.t. } &z_{ij} \leq z_{jj}, \quad i \in [N]; j \notin M; i \neq j\label{LinKMed2-2}\\
&\sum_{i\notin {\cal M}} z_{ii} - \sum_{i \in {\cal M}} \sum_{j \neq i} z_{ij} \leq 0 \label{LinKMed2-2a}\\
&x_{ij} \leq 1 - \sum_{\ell \neq j} z_{j\ell }, \quad j \in M; i \notin S_j \label{LinKMed3-2}\\
&\sum_{\ell \neq M(i)} z_{M(i),\ell} \leq \sum_{\ell \neq M(i)} z_{i\ell}, \quad i \notin M \label{LinKMed4-2}\\
&\sum_{\ell \neq M(i)} z_{i\ell} \leq 1, \quad i \in [N] \label{LinKMed5-2} \\
& z_{ij} \geq 0, \quad (i,j) \in \Omega^c\label{LinKMed6-2}\,,
\end{align}
where $p_{ij}\equiv w_{ij} - w_{i, M(i)}$. The only $z_{ij}$ in the program (\ref{LinKMedObj-2})-(\ref{LinKMed6-2}) have $(i, j) \in \Omega^c$. Associate the \textit{nonnegative} dual variables $\theta_{ij}$, $u$, $\gamma_{ij}$, $\lambda_{i}$, $s_i$, and $L_{ij}$ with (\ref{LinKMed2-2}), (\ref{LinKMed2-2a}), (\ref{LinKMed3-2}), (\ref{LinKMed4-2}), (\ref{LinKMed5-2}), and (\ref{LinKMed6-2}), respectively. Enforcing stationarity of the Lagrangian gives
\begin{align}
p_{ij} - u + \theta_{ij} + \sum_{\ell \in S_i, \ell \notin {\cal M}} \lambda_{\ell} + s_i - L_{ij} = 0, \quad i \in {\cal M}; j \notin {\cal M} \label{KKT1}\\
p_{ij} + \theta_{ij} - \lambda_i + s_i - L_{ij} =  0, \quad i \notin {\cal M}; j \notin {\cal M}; i \neq j \label{KKT2}\\
p_{jj} + u - \sum_{\ell \neq j}^N \theta_{\ell j} - \lambda_j + s_i - L_{jj} =  0, \quad j \notin {\cal M}; i = j \label{KKT3}\\
p_{ij} - u + \gamma_{ij} + \gamma_{ji} + \sum_{\ell \in S_i, \ell \notin {\cal M}} \lambda_{\ell} + s_i - L_{ij} =  0, \quad i \in {\cal M}; j \in {\cal M}, i \neq j \label{KKT4}\\
p_{ij} - \lambda_i + \gamma_{ij} + \gamma_{ji} + s_i - L_{ij} = 0, \quad i \notin {\cal M}; j \in {\cal M}; j \neq M(i). \label{KKT5}
\end{align}

Call the primal Lagrangian $f(\mathbf{z}_{\Omega^c})$.  Above, the quantities on the lefthand sides of the equalities are components of $\nabla_{\mathbf{z}_{\Omega^c}}f(\mathbf{z}_{\Omega^c})$. Because ${\bf z}_{\Omega^c}^\# = \mathbf{0}$, complementary slackness of (\ref{LinKMed3-2}) and (\ref{LinKMed5-2}) gives that $\gamma_{ij} = 0$ and $s_i = 0$ where a medoid solution is exactly recovered. The $L_{ij}$ are merely slack variables. Uniqueness of the solution $\mathbf{z}_{\Omega^c} = \mathbf{z}_{\Omega^c}^{\#}$ occurs if and only if for any feasible perturbation $\mathbf{h}_{\Omega^c}$ of $\mathbf{z}_{\Omega^c}^\#$, 
$$
f(\mathbf{z}_{\Omega^c}^{\#} + \mathbf{h}_{\Omega^c}) \geq f(\mathbf{z}_{\Omega^c}^{\#}) + \scalprod{\nabla_{\mathbf{z_{\Omega^c}}}f(\mathbf{z}_{\Omega^c}^{\#})}{\mathbf{h}_{\Omega^c}} > f(\mathbf{z}_{\Omega^c}^{\#})\,.
$$
Because the feasible solution set includes only nonnegative $\mathbf{z}_{\Omega^c}^\#$, any feasible perturbation $\mathbf{h}_{\Omega^c}$ away from $\mathbf{z} = \mathbf{z}_{\Omega^c}^{\#} = \mathbf{0}$ must be nonnegative with at least one positive component. Demanding that every component of $\nabla_{\mathbf{z}_{\Omega^c}}f(\mathbf{z}_{\Omega^c})$---that is, each LHS of \eqref{KKT1}-\eqref{KKT5}---is positive thus simultaneously satisfies the KKT conditions and guarantees solution uniqueness. More precisely, \textsc{LinKMed} has a unique solution $\mathbf{z} = \mathbf{z}^\#$ that coincides with the solution to \textsc{KMed} if and only if there exist $u, \lambda_i, \theta_{ij} \in \mathbb{R}$ that satisfy
\begin{align}
p_{ij} - u + \theta_{ij} + \sum_{\ell \in S_i, \ell \notin {\cal M}} \lambda_{\ell} > 0, \quad i \in {\cal M}; j \notin {\cal M} \label{KKT1strict}\\
p_{ij} + \theta_{ij} - \lambda_i > 0, \quad i \notin {\cal M}; j \notin {\cal M}; i \neq j \label{KKT2strict}\\
p_{jj} + u - \sum_{\ell \neq j}^N \theta_{\ell j} - \lambda_j > 0, \quad j \notin {\cal M}; i = j \label{KKT3strict}\\
p_{ij} - u + \sum_{\ell \in S_i, \ell \notin {\cal M}} \lambda_{\ell} > 0, \quad i \in {\cal M}; j \in {\cal M}, i \neq j \label{KKT4strict}\\
p_{ij} - \lambda_i > 0, \quad i \notin {\cal M}; j \in {\cal M}; j \neq M(i). \label{KKT5strict}
\end{align}

Assigning each $\theta_{ij}$ its minimum possible value minimizes the restrictiveness of (\ref{KKT3strict}). From (\ref{KKT1strict}), the minimum possible value of $\theta_{i \in M, j}$ approaches $\left(u- p_{ij}-\sum_{k \in S_i, \ell \notin {\cal M}} \lambda_{\ell} \right)_+$ from above. From (\ref{KKT2strict}), the minimum possible value of $\theta_{i\notin M, j}$ approaches $\left(\lambda_i - p_{ij}\right)_+$ from above. Inserting these values of $\theta_{ij}$ into the conditions above gives 
\begin{align}
p_{jj} + u - \lambda_j - \sum_{i \notin {\cal M}, i \neq j} \left(\lambda_i - p_{ij}\right)_+ +\sum_{i \in {\cal M}} \left(u - p_{ij}-\sum_{\ell \in S_i, \ell \notin {\cal M}} \lambda_{\ell} \right)_+ > 0, \quad j \notin {\cal M} \label{KKT1strict1}\\
\lambda_i < p_{ij},\quad i \notin {\cal M}; j \in {\cal M} \label{KKT2strict1}\\
 u - \sum_{\ell \in S_i, \ell \notin {\cal M}} \lambda_{\ell} < p_{ij}, \quad i \in {\cal M}; j \in {\cal M}, i \neq j \label{KKT3strict1}\,.
\end{align}
There exist $u, \lambda_i$ that satisfy \eqref{KKT1strict1}-\eqref{KKT3strict1} if and only if there exist $u, \lambda_i, \theta_{ij}$ that satisfy \eqref{KKT1strict}-\eqref{KKT5strict}.

Since $\lambda_j - p_{jj}$ is nonnegative, it can be absorbed into the sum over $i \notin {\cal M}$ in \eqref{KKT1strict1}:
$$
 u - \sum_{i \notin {\cal M}} \left(\lambda_i - p_{ij}\right)_+ +\sum_{i \in {\cal M}} \left(u - p_{ij}-\sum_{\ell \in S_i, \ell \notin {\cal M}} \lambda_{\ell} \right)_+ > 0, \quad j \notin {\cal M}\,.
$$
Define
$$
\lambda_{i} = u - \sum_{\ell \in S_i, \ell \notin {\cal M}} \lambda_{\ell} \,,\quad i \in {\cal M}
$$
to recover the content of the proposition.
\end{proof}

\end{document}